\documentclass[sigconf]{acmart}
\usepackage{bm}
\usepackage{amsfonts}
\usepackage{textcomp}

\usepackage[most]{tcolorbox}
\tcbuselibrary{breakable}
\usepackage[table,xcdraw]{xcolor}
\usepackage{tabularx}
\usepackage[utf8]{inputenc}
\usepackage[linesnumbered,ruled,vlined]{algorithm2e}
\usepackage{soul}
\usepackage{float}
\usepackage{enumitem}
\usepackage{lipsum}   

\newcommand{\methodName}{SymCA}
\newcommand{\reveal}{REVEAL+}
\newcommand{\tabemb}{TabEmb}

\usepackage{url}

\usepackage{amsthm}
\usepackage{multirow}
\usepackage{makecell}
\usepackage{subcaption}
\usepackage{listings}
\usepackage{xcolor}
\newsavebox{\promptdesignbox}

\lstdefinestyle{prompt}{
  basicstyle=\ttfamily\scriptsize,
  breaklines=true,
  breakatwhitespace=false,
  columns=fullflexible,
  keepspaces=true,
  showstringspaces=false,
  frame=single,
  rulecolor=\color{gray!55},
  backgroundcolor=\color{gray!7},
  xleftmargin=5pt,
  xrightmargin=3pt,
  aboveskip=4pt,
  belowskip=2pt,
  breakindent=8pt,
}

\usepackage[utf8]{inputenc}
\usepackage{colortbl}
\usepackage[linesnumbered,ruled,vlined]{algorithm2e}

\usepackage{pifont}

\usepackage{enumitem}

\newtheorem{definition}{\textbf{Definition}}

\newtheorem{example}{\textbf{Example}}

\newtheorem{lemma}{\textbf{Lemma}}

\makeatletter
\patchcmd{\proof}
  {\topsep}
  {3pt \@plus 0.1ex \@minus 0.1ex}
  {}{}
\patchcmd{\endproof}
  {\topsep}
  {3pt \@plus 0.1ex \@minus 0.1ex}
  {}{}
\makeatother

\setcopyright{none}
\renewcommand\footnotetextcopyrightpermission[1]{}

\begin{document}

\title{Interpretable Column Annotation with LLM-Symbolized Decision Process Materialization}
\author{Mengqi Wang}
    \orcid{0009-0008-2900-7413}
    \affiliation{%
      \institution{University of New South Wales}
      \city{Sydney}
      \state{NSW}
      \country{Australia}
    }
    \email{mengqi.wang4@unsw.edu.au} 
    
\author{Jianwei Wang}
    \orcid{0009-0000-7887-4179}
    
    \affiliation{%
      \institution{University of New South Wales}
      \city{Sydney}
      \state{NSW}
      \country{Australia}
    }
    \email{jianwei.wang1@unsw.edu.au}
    \authornote{Jianwei Wang is the corresponding author. }
    
\author{Qing Liu}
    \orcid{0000-0001-7895-9551}

    \affiliation{%
      \institution{Data61, CSIRO}
      \city{Sandy Bay}
      \state{Tasmania}
      \country{Australia}
    }
 \email{q.liu@data61.csiro.au}
    
\author{Xiwei Xu}
  \orcid{0000-0002-2273-1862}
  \affiliation{%
      \institution{Data61, CSIRO}
        \city{Eveleigh}
        \state{NSW}
        \country{Australia}
      }
  \email{xiwei.xu@data61.csiro.au}

\author{Zhenchang Xing}
  \orcid{0000-0001-7663-1421}
  \affiliation{%
      \institution{Data61, CSIRO}
        \city{Canberra}
        \state{ACT}
        \country{Australia}
      }
  \email{zhenchang.xing@data61.csiro.au}

\author{Michael Bain}
    \orcid{0000-0002-4309-6511}
    \affiliation{%
      \institution{University of New South Wales}
      \city{Sydney}
      \state{NSW}
      \country{Australia}
    }
    \email{m.bain@unsw.edu.au}

\author{Liming Zhu}
\orcid{0000-0001-5839-3765}
  \affiliation{%
      \institution{Data61, CSIRO}
        \city{Sydney}
        \state{NSW}
        \country{Australia}
      }
  \email{liming.zhu@data61.csiro.au}

\author{Wenjie Zhang}
    \orcid{0000-0001-6572-2600}
    \affiliation{%
      \institution{University of New South Wales}
      \city{Sydney}
      \state{NSW}
      \country{Australia}
    }
    \email{wenjie.zhang@unsw.edu.au}

\renewcommand{\shortauthors}{Mengqi Wang et al.}

\begin{abstract}
Column annotation (CA), including column type annotation (CTA) and column property annotation (CPA), aims to identify the meanings of table columns and the semantic relationships among them. 
Recent CA methods usually use various neural models to learn column representations and directly map them to label categories, thereby (1) sacrificing model interpretability and adaptivity, (2) overlooking rich label semantics and ultimately limiting accuracy.
To address these limitations, we propose \methodName{}, an LLM-empowered interpretable CA framework that materializes column annotation as a global-to-local symbolic decision process. \methodName{} consists of two components: (1) global skeleton induction, which constructs a semantic skeleton over the label space, and (2) local substrate evolution, which evolves predictive substrates within the skeleton.
Specifically, to exploit label semantics while preserving an interpretable decision process, the global skeleton induction module leverages LLMs to generate candidate hypernym-inspired tree-structured semantic skeletons and employs a Minimum Bayes Risk (MBR)-based consensus strategy to select a robust skeleton against generation variance.
Since different internal nodes require different evidence to distinguish among their child nodes, the local substrate evolution module materializes each internal node as an executable and evolvable predictive substrate. Over multiple evolution rounds, each substrate trains an interpretable random forest classifier with the current operator set, leverages the LLM to propose node-specific operator modifications, and uses an exploration--exploitation strategy to prioritize promising substrates.
Extensive experiments demonstrate that \methodName{} is accurate, robust, and interpretable, outperforming the strongest baselines by an average of 6.42\% in Micro-F1 and 11.03\% in Macro-F1.

\end{abstract}


    \maketitle

\begin{figure}[t]
  \centering
  \includegraphics[width=0.98\linewidth]{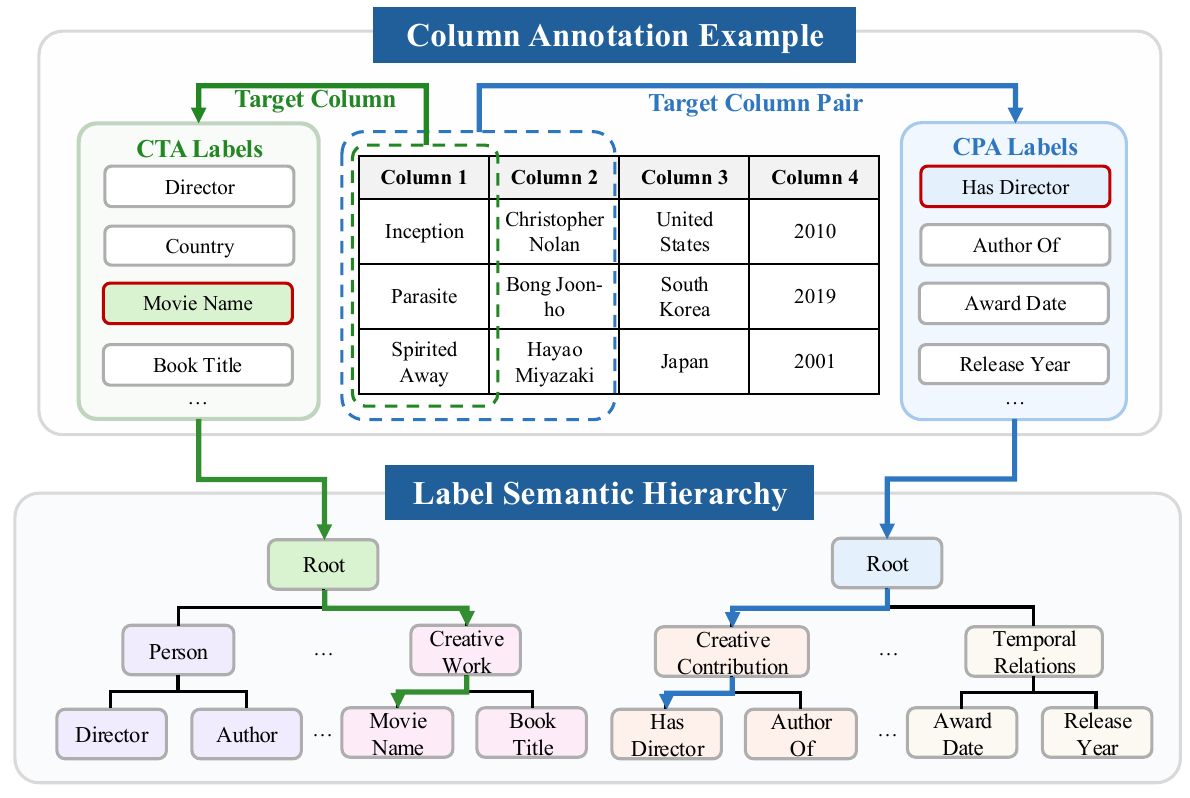}
    \caption{A motivating example of column annotation with a hypernym-inspired label semantic hierarchy.}
\label{fig:example}
\end{figure}

\section{Introduction}
\label{sec:Introduction}

Relational tables are a key format for organizing structured data across databases, spreadsheets, web tables, and enterprise data repositories. Their effective management and use across heterogeneous applications depend heavily on semantic metadata, such as column types and relationships between columns, which is often missing, incomplete or ambiguous in practice~\cite{chen2019colnet, ding2025retrieve}. Column annotation (CA) seeks to infer missing semantic metadata from table contents by mapping columns and their relationships to semantic classes and properties defined in widely recognized ontologies or vocabularies, such as DBpedia~\cite{lehmann2015dbpedia} and Schema.org~\cite{guha2016schema}. Two central CA tasks are column type annotation (CTA), which assigns a semantic type to a target column, and column property annotation (CPA), which identifies the semantic relation between an ordered pair of columns~\cite{deng2022turl,suhara2022annotating,miao2023watchog}. These annotations support downstream applications such as data integration, schema matching, table search, and data repair~\cite{deng2022turl,suhara2022annotating,miao2023watchog,liu2026hyperjoin, chen2026collaborative,chen2026empowering,sun2026lakehopper}.

\begin{example}
Figure~\ref{fig:example} illustrates the two main tasks in CA using a movie table. CTA assigns a semantic type to each target column; for example, the first column is annotated as \texttt{Movie Name}. CPA identifies the directed semantic relation between a pair of columns; for example, the relation from the first column to the second column is \texttt{Has Director}. These labels are semantically related rather than independent. For instance, \texttt{Person} is a hypernym of \texttt{Director} and \texttt{Author}, while \texttt{Temporal Relation} is a hypernym of \texttt{Award Date} and \texttt{Release Year}. Explicitly modeling these hypernym-inspired relations helps enable an interpretable decision process.
\end{example}

Given the importance of CA, a set of methods has been developed. 
Traditional knowledge-base-driven CA methods derive semantic annotations for individual columns and column pairs by linking cell values to entities in a reference knowledge base and exploiting the associated type and relation information~\cite{limaye2010annotating, efthymiou2017matching, luo2018cross}. However, suitable knowledge bases may be unavailable or provide limited coverage for real-world relational tables.
The advancement of language models has opened a new paradigm for CA. Through large-scale pretraining, language models acquire broad lexical and semantic knowledge and can contextualize heterogeneous table values, enabling them to infer latent column meanings even when explicit entity links or schema alignments are unavailable. Building on these capabilities, recent CA studies either infer semantic annotations directly through prompting large language models (LLMs)~\cite{feuer2024archetype,korini2023column,korini2024column} or learn contextual representations from table contents for downstream prediction~\cite{suhara2022annotating,ding2025retrieve,hoseinzade2026tabemb}.

\noindent \textbf{Motivations.} Existing methods still face two core limitations. 
First, existing methods provide limited interpretability and adaptivity. They typically encode table contents into latent representations and directly map them to label categories through globally parameterized predictors. As a result, the semantic evidence and intermediate decisions underlying each prediction remain largely implicit. Moreover, their predictive behavior is primarily improved through parameter optimization within a predefined representation, making it difficult to adapt the decision process to the semantic evidence required by different label groups.
Second, existing methods overlook the rich semantics of the label space and thus limit the accuracy of CA. They commonly treat labels as independent class identifiers, although many labels share broader semantic concepts and exhibit meaningful hierarchical relationships. Ignoring such structure prevents the prediction from being decomposed into interpretable semantic decisions and makes it more difficult to distinguish closely related labels within their appropriate semantic contexts.

\noindent \textbf{Challenges.}
Two challenges remain in leveraging rich label semantics to build an accurate, interpretable and adaptive CA method.

\textit{Challenge 1: Complexity of the label semantic space.}
The complexity of label semantics makes it challenging to design interpretable and accurate CA methods.
Label vocabularies are often inherited from heterogeneous ontology classes (as illustrated in Figure~\ref{fig:example}) and properties associated with Web tables, rather than designed as task-specific taxonomies. 
As a result, labels vary in semantic domain, ontological role, abstraction level, and annotation scope, while their boundaries and interrelations remain uneven or only partially explicit, leading to overlapping scopes, cross-domain dependencies, and non-uniform hierarchical structures.

\textit{Challenge 2: Label ambiguity in complex table contexts.}
The ambiguity of labels in complex table contexts makes it challenging to design interpretable and adaptive CA methods.
The same label may be instantiated by substantially different values across domains, sources, and representation conventions, whereas identical or similar values may correspond to different labels depending on their surrounding columns and relational roles. Consequently, neither individual values nor globally shared value patterns can uniquely determine the intended semantics. This ambiguity is further amplified when informative metadata, such as headers or table descriptions, is unavailable, leaving cross-column value patterns as the primary evidence for disambiguation. Thus, identifying the table-specific evidence that connects an observed value pattern to the appropriate semantic label remains a fundamental challenge.

\noindent \textbf{Our approaches.} To address these challenges, we propose \methodName{}, an LLM-empowered interpretable CA framework that materializes column annotation as a global-to-local symbolic decision process. \methodName{} comprises two stages: (1) \emph{global skeleton induction}, which organizes label semantics into a hierarchical skeleton; and (2) \emph{local substrate evolution}, which materializes this skeleton into an adaptive symbolic predictor through node-specific substrate evolution.

At the global level, to exploit the complex semantic knowledge embedded in labels while preserving decision interpretability, \methodName{} leverages the extensive world knowledge and reasoning capabilities of pretrained LLMs to infer hypernym-inspired relations and organize the labels into a tree-structured semantic skeleton, thereby decomposing global annotation into a sequence of localized semantic decisions. To reduce generation variance and prompt sensitivity, \methodName{} constructs multiple candidate skeletons through repeated LLM inference and applies a Minimum Bayes Risk (MBR)-based consensus strategy~\cite{kumar2004minimum} to select a robust skeleton with the lowest expected structural disagreement across the candidates. The resulting skeleton provides an explicit symbolic representation of the semantic structure of the label space.

At the local level, to turn the induced semantic skeleton into an executable CA predictor, we equip each internal node with a predictive substrate that performs node-specific semantic classification over its child nodes. To make this decision process interpretable, each substrate employs a random-forest classifier~\cite{breiman2001random} over explicit symbolic operators that capture semantic conditions specific to the corresponding node. Because different substrates require different evidence to distinguish among their child nodes, each substrate evolves its operator set through three stages. \methodName{} first trains and evaluates the current substrates to identify deficiencies in their local predictions, then leverages an LLM to propose node-specific operator modifications, and finally employs an exploration--exploitation strategy to prioritize promising substrates for efficient retraining.

In summary, this paper makes the following contributions:
\begin{itemize}
\item We introduce an LLM-empowered global-to-local CA framework for interpretable and adaptive column annotation, which materializes the annotation decision process through global skeleton induction and local substrate evolution.

\item We develop a global skeleton induction method that materializes implicit label semantics into a robust hypernym-inspired symbolic skeleton through multiple LLM-induced skeleton candidates and MBR-based consensus selection.

\item We develop a local substrate evolution method that materializes the induced skeleton into an evolvable symbolic predictor through random-forest-based substrate training, LLM-guided substrate modification, and exploration--exploitation-based substrate selection.

\item We conduct extensive experiments demonstrating that \methodName{} delivers accurate, robust, and interpretable annotation, outperforming the strongest baseline by an average of 6.42\% in Micro-F1 and 11.03\% in Macro-F1.

\end{itemize}

The source code for this paper is available at \url{https://github.com/T-Lab/SymCA}.
\section{Problem Statement}
\label{sec:problem}

Let $T = (c_1, c_2, \ldots, c_n)$ be a table with $n$ columns, where $c_i$ denotes the $i$-th column in the table. Each column consists of a
sequence of cell values, $c_i = (z_i^1, z_i^2, \ldots, z_i^m)$, where $z_i^j$ is the value of the $j$-th cell in column $c_i$. Following SOTA CA methods~\cite{ding2025retrieve}, we assume that metadata such as column headers and table names are unavailable, and infer semantic annotations solely from table contents.

\begin{definition}[Column Type Annotation]
Given a table $T$, a target column $c_i$, and a predefined column-type label
space $\mathcal{L}^{\mathrm{CTA}}$, column type annotation aims to predict the
semantic type of $c_i$. The task is to learn a model $F^{\mathrm{CTA}}(T, c_i) \in \mathcal{L}^{\mathrm{CTA}}$.
\end{definition}

\begin{definition} [Column Property Annotation]
Given a table $T$, an ordered target column pair $(c_i,c_j)$, and a predefined
column-property label space $\mathcal{L}^{\mathrm{CPA}}$, column property
annotation aims to predict the semantic property from $c_i$ to $c_j$. The task is to learn a model $F^{\mathrm{CPA}}(T, c_i, c_j) \in \mathcal{L}^{\mathrm{CPA}}$.
\end{definition}

\begin{figure*}[t]
  \centering
  \includegraphics[width=0.98\linewidth]{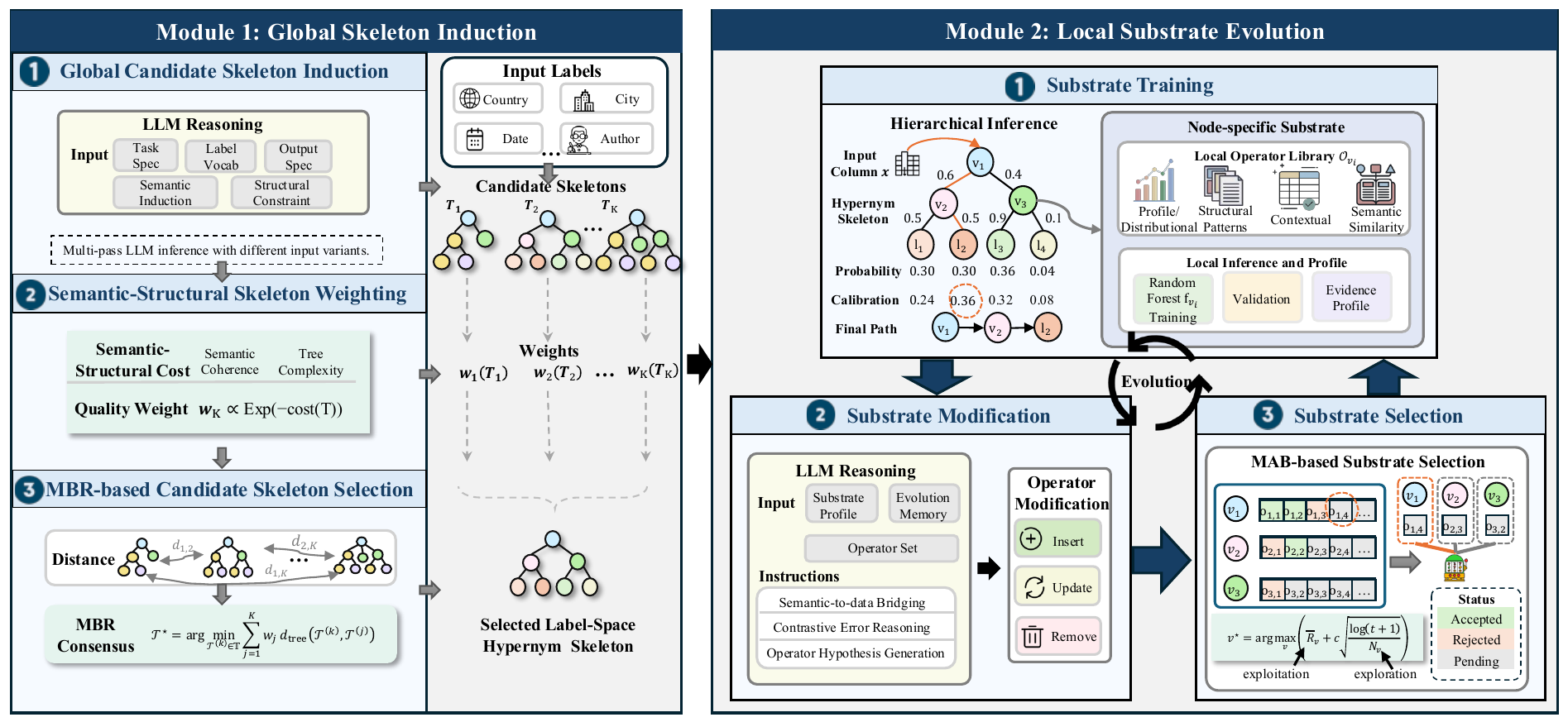}
  \caption{Overview of \methodName{}. Module~1 induces a global semantic skeleton over the label space, while Module~2 materializes its internal nodes as evolvable local predictive substrates.}
  \label{fig:frameworks_overview}
  \vspace{-2mm}
\end{figure*}

\section{Methodology}
\label{sec:methodology}

Figure~\ref{fig:frameworks_overview} presents \methodName{}, an interpretable CA framework that materializes CA as a global-to-local symbolic decision process. Module~1 transforms an unstructured label vocabulary into a hypernym-inspired semantic skeleton (Section~\ref{sec:module1}). Guided by this global structure, Module~2 defines each internal node as a local predictive substrate and progressively refines its operators (Section~\ref{sec:module2}). The overall workflow is summarized in Algorithms~\ref{app:algorithms}.

\subsection{Global Skeleton Induction}
\label{sec:module1}

\noindent \textbf{Motivation.} 
LLMs provide rich prior semantic knowledge and compositional reasoning capabilities, enabling them to leverage world knowledge to interpret task-specific semantics and organize them into structured symbolic representations, such as trees and graphs~\cite{wang2025ensembling,knauer2025oh,chen2024sac,mo2026kggen,xu2026self,tan2026privgemo, jiang2026advancing}. Leveraging these capabilities, we employ an LLM to induce a tree-structured skeleton whose hierarchical organization naturally captures hypernym-inspired relations among labels with varying semantic granularities, thereby materializing implicit label semantics as an explicit semantic hierarchy.
To mitigate structural uncertainty and generation variability, we adopt a three-stage induction process to obtain a reliable skeleton: (1) \emph{Global Candidate Skeleton Induction}, which uses prompt variants to propose alternative hypernym-inspired skeletons; (2) \emph{Semantic-Structural Skeleton Weighting}, which evaluates the semantic coherence and structural quality of each candidate; and (3) \emph{MBR-Based Candidate Skeleton Selection}, which selects a high-quality candidate that is most consistent with the remaining candidates.

\subsubsection{Global Candidate Skeleton Induction}
\label{sec:m1:candidates}
To obtain diverse candidate skeletons for subsequent evaluation and consensus, we prompt the LLM multiple times under mild, semantics-preserving perturbations of the same induction request. Each inference integrates the task description, label texts, and structural instructions to infer latent semantic relations within the label space and organize them into a tree-structured skeleton (detailed in Appendix~\ref{app:prompts}). Each prompt contains five components: (1) a \emph{task specification} describing the CA task and the intended role of the skeleton; (2) a \emph{label vocabulary} containing the complete target label set $\mathcal{L}$; (3) a \emph{semantic induction objective} together with task-informed grouping guidelines, requiring the LLM to identify hypernym-inspired relations, group labels purely by the meanings of their names, and assign every internal node a meaningful concept name; (4) a set of \emph{structural constraints}, including a single complete tree, exact label coverage, unique leaf placement, bounded depth and branching, and no single-child internal nodes; and (5) an \emph{output specification} requiring a JSON structured tree with an example.

To induce multiple alternatives without changing the underlying semantic objective, we construct prompt variants that preserve the same labels, induction objective, and structural constraints while perturbing only their surface realization. Specifically, the variants permute the label ordering, reorder the requirement statements, paraphrase each requirement without changing its meaning, and replace the worked example. We invoke the LLM independently with each variant and treat every valid output as a candidate skeleton.

\subsubsection{Semantic-Structural Skeleton Weighting}
\label{sec:m1:scoring}

The generated candidates are structurally valid, but they may differ in how they capture label semantics and how compactly they organize the label space. We therefore assign each candidate a semantic-structural weight that reflects both local semantic coherence and global structural complexity. Specifically, we first compute a cost that penalizes semantically related labels placed in broad subtrees and skeletons with inappropriate granularity, and then convert this cost into a normalized quality weight. A higher score therefore indicates a candidate that organizes related labels into compact local subspaces without making the overall skeleton unnecessarily complex.

To quantify the semantic coherence of a candidate skeleton, we use a pairwise label-affinity cost derived from Dasgupta's similarity-based hierarchical clustering objective~\cite{dasgupta2016cost}. Specifically, we embed each label name $\ell_i$ using the pretrained \texttt{all-MiniLM-L6-v2} sentence encoder~\cite{wang2020minilm,reimers2019sentence}, which provides efficient and semantically informative representations for short texts. Let $\mathbf{e}_i$ denote the resulting L2-normalized embedding. We then define the semantic affinity between two labels $\ell_i,\ell_j\in\mathcal{L}$ as the clipped cosine similarity of their embeddings
$\sigma_{ij}
=
\max\left(0,\mathbf{e}_i^\top\mathbf{e}_j\right)$
with $\sigma_{ii}=1$. This yields a non-negative affinity score in $[0,1]$, where negatively correlated embeddings are treated as semantically unrelated.
For a candidate skeleton $\mathcal{T}$, let $\mathrm{LCA}_{\mathcal{T}}(\ell_i,\ell_j)$ be the lowest common ancestor of $\ell_i$ and $\ell_j$. We define the semantic cost as
\[
    \mathrm{Cost}_{\mathrm{sem}}(\mathcal{T})
    =
    \sum_{i<j}
    \sigma_{ij}
    \cdot
    \left|
    \mathrm{Leaves}
    \left(
    \mathrm{LCA}_{\mathcal{T}}(\ell_i,\ell_j)
    \right)
    \right|.
\]
This cost penalizes cases where semantically similar labels are only connected through a high-level ancestor with many descendant labels. Thus, lower semantic cost favors skeletons that place related labels under more specific shared semantic nodes.

Semantic coherence alone does not control the granularity of the skeleton. A tree with too few internal nodes yields a weakly decomposed hierarchy, whereas one with too many internal nodes may fragment the data into undersized local classification problems. We therefore introduce a granularity regularizer based on the number of internal nodes:
\[
    \mathrm{Cost}_{\mathrm{str}}(\mathcal{T})
    =
    \lambda \cdot \overline{\mathrm{Cost}}_{\mathrm{sem}} \cdot
    \frac{\left||\mathcal{V}_{\mathrm{int}}(\mathcal{T})| - \kappa\right|}{\kappa},
\]
where $\mathcal{V}_{\mathrm{int}}(\mathcal{T})$ is the set of internal nodes of $\mathcal{T}$, $\kappa$ is a target budget proportional to $|\mathcal{L}|$, $\overline{\mathrm{Cost}}_{\mathrm{sem}}$ is the mean semantic cost over all candidates used to rescale the structural penalty, and $\lambda>0$ is a fixed weight controlling the overall strength of this regularizer. The total cost is
\[
    \mathrm{Cost}(\mathcal{T})
    =
    \mathrm{Cost}_{\mathrm{sem}}(\mathcal{T})
    +
    \mathrm{Cost}_{\mathrm{str}}(\mathcal{T}).
\]

Finally, to obtain a positive and concentration-controllable distribution that monotonically favors lower-cost candidates, we transform the candidate costs into normalized exponential weights~\cite{dalalyan2008aggregation}:
\begin{equation}
\label{eq:reliability_weights}
    w_k
    =
    \frac{
        \exp(-\eta \, \mathrm{Cost}(\mathcal{T}^{(k)}))
    }{
        \sum_{j=1}^{K}
        \exp(-\eta \, \mathrm{Cost}(\mathcal{T}^{(j)}))
    },
\end{equation}
where $\eta>0$ controls the concentration of the weighting distribution. Candidates with lower costs receive larger weights, which summarize their relative semantic-structural quality.

\subsubsection{MBR-Based Candidate Skeleton Selection}
\label{sec:m1:mbr} 
Section~\ref{sec:m1:scoring} computes semantic-structural weights that assess the quality of individual candidate skeletons but do not capture their agreement with other candidates. We therefore use these weights within a distribution-level consensus criterion that selects a skeleton that is both individually coherent and structurally representative of other high-quality candidates. Specifically, we treat the semantic-structural weights as an empirical distribution over the candidates, measure pairwise structural disagreement using a Jaccard-normalized Robinson--Foulds distance~\cite{robinson1981comparison}, and apply minimum Bayes risk (MBR) decoding~\cite{kumar2004minimum} to select the candidate with the minimum expected structural disagreement under this distribution.

Let $\mathbb{T}=\{\mathcal{T}^{(1)},\ldots,\mathcal{T}^{(K)}\}$ denote the candidate set. The normalized weights from Eq.~\eqref{eq:reliability_weights} define an empirical distribution $\pi_{\eta}$ over the candidates, where $\pi_{\eta}(\widetilde{\mathcal{T}}=\mathcal{T}^{(j)})=w_j$. Under this distribution, candidates with higher semantic-structural quality contribute more strongly to the consensus. 
Each internal node of a skeleton groups the labels beneath it into one semantic cluster; we call that label set a \emph{clade}. Formally, the clade induced by an internal node is the set of all leaf labels in the subtree rooted at that node, and $\mathrm{Clades}(\mathcal{T})$ collects the clades of every non-root internal node of $\mathcal{T}$.
Since a rooted tree is determined by its clade set, comparing two skeletons reduces to comparing $\mathrm{Clades}(\mathcal{T})$ and $\mathrm{Clades}(\mathcal{T}')$.
For two candidates $\mathcal{T}$ and $\mathcal{T}'$, we adopt a Jaccard-normalized variant of the Robinson--Foulds distance~\cite{robinson1981comparison}:
\[
d_{\mathrm{tree}}(\mathcal{T},\mathcal{T}')
=
\frac{
\left|\mathrm{Clades}(\mathcal{T})\triangle\mathrm{Clades}(\mathcal{T}')\right|
}{
\left|\mathrm{Clades}(\mathcal{T})\cup\mathrm{Clades}(\mathcal{T}')\right|
},
\]
where $\triangle$ denotes the symmetric difference. The numerator counts the clades appearing in only one of the two skeletons, while the denominator normalizes this disagreement by the total number of distinct clades appearing in either skeleton. Thus, $d_{\mathrm{tree}}\in[0,1]$, with smaller values indicating stronger structural agreement.
We then select the candidate that minimizes its expected structural disagreement with the quality-weighted candidate distribution:
\begin{equation}
\label{eq:mbr_medoid}
\mathcal{T}^{\star}
=
\arg\min_{\mathcal{T}^{(k)}\in\mathbb{T}}
\sum_{j=1}^{K}
w_j\,
d_{\mathrm{tree}}
\left(
\mathcal{T}^{(k)},
\mathcal{T}^{(j)}
\right).
\end{equation}
Equivalently, $\mathcal{T}^{\star}$ is the quality-weighted medoid of the candidate set, balancing its individual semantic-structural quality against its structural agreement with other high-quality candidates. This criterion is designed to base the final decision on the collective evidence of the candidate pool, thereby reducing sensitivity to an individually favorable but structurally idiosyncratic generation. To formalize this property, define
$\mathrm{Risk}_{\eta}(\mathcal{T})
=
\mathbb{E}_{\widetilde{\mathcal{T}}\sim\pi_{\eta}}
\left[
d_{\mathrm{tree}}(\mathcal{T},\widetilde{\mathcal{T}})
\right]$
as the objective in Eq.~\eqref{eq:mbr_medoid}, and let
$D_{\mathbb{T}}
=
\max_{\mathcal{T},\mathcal{T}'\in\mathbb{T}}
d_{\mathrm{tree}}(\mathcal{T},\mathcal{T}')
\leq 1$
denote the diameter of the candidate set. The following lemma characterizes the resulting consensus guarantee.

\begin{lemma}[Reliability of consensus selection]
\label{prop:mbr_medoid}
Suppose that some $\bar{\mathcal{T}}\in\mathbb{T}$ has weighted
$\delta$-neighborhood mass of at least $\rho$, i.e.,
$
\pi_{\eta}
\left(
d_{\mathrm{tree}}
\left(
\bar{\mathcal{T}},
\widetilde{\mathcal{T}}
\right)
\leq \delta
\right)
\geq \rho,
$
for some $0\leq\delta\leq D_{\mathbb{T}}$ and $\rho\in(0,1]$. Then:
\begin{enumerate}
    \item
    $\mathrm{Risk}_{\eta}(\mathcal{T}^{\star})
    \leq
    \rho\delta+(1-\rho)D_{\mathbb{T}}$;
    \item
    no candidate $\mathcal{T}'$ satisfying
    $
    d_{\mathrm{tree}}
    \left(
        \mathcal{T}',
        \bar{\mathcal{T}}
    \right)
    >
    2\delta+\frac{1-\rho}{\rho}D_{\mathbb{T}}
    $
    can be selected by Eq.~\eqref{eq:mbr_medoid}.
\end{enumerate}
\end{lemma}

Part~(i) shows that concentration of the quality-weighted candidates around a shared structure bounds the expected structural disagreement of the selected skeleton. Part~(ii) gives a complementary exclusion guarantee: a candidate sufficiently far from this concentrated mass cannot minimize the MBR objective, even with a favorable individual score. Together, these results limit the influence of structurally idiosyncratic generations on consensus selection (proof in Appendix~\ref{app:mbr_medoid}).

\subsection{Local Substrate Evolution}
\label{sec:module2}

\noindent \textbf{Motivation.} Module~1 provides a symbolic skeleton of the label space, but the skeleton alone does not specify how each internal node should make a semantic decision among its child nodes. A straightforward solution is to equip each internal node with a classifier and perform hierarchical inference along the skeleton. However, this design is insufficiently adaptive because different nodes require heterogeneous semantic evidence that cannot be adequately captured by a uniform, predefined operator space. We therefore materialize each internal node as an evolvable predictive substrate that is progressively optimized over multiple evolution rounds by modifying its node-specific operators. Each evolution round comprises three stages: (1) \emph{Substrate Training}, which trains the current classifiers and identifies evidence of semantic distinctions insufficiently captured by their operators; (2) \emph{Substrate Modification}, which leverages an LLM to propose substrate operator modifications based on the evidence; and (3) \emph{Substrate Selection}, which leverages an exploration--exploitation strategy to efficiently prioritize the most promising substrates for modification evaluation.

\subsubsection{Substrate Training}
\label{sec:m2:training}
Each internal node defines a local semantic classification problem over its child nodes and therefore requires a predictive mechanism tailored to the corresponding semantic context. Substrate training operationalizes the current substrate state by training a node-specific classifier on the evidence produced by its operator set, while profiling prediction deficiencies that guide subsequent operator modification.
At evolution round $r$, the substrate associated with each internal node $v$ is represented as
\[
\mathcal{S}_v^{(r)}
=
\left(
\mathcal{O}_v^{(r)},
f_v^{(r)},
\mathcal{H}_v^{(<r)}
\right),
\]
where $\mathcal{O}_v^{(r)}$ is its node-specific symbolic operator set, $f_v^{(r)}$ is a random-forest classifier~\cite{breiman2001random} that makes local decisions among the child nodes of $v$, and $\mathcal{H}_v^{(<r)}$ stores previously proposed operator modifications and their validated outcomes.
An operator is a deterministic function that maps a CTA target column or a CPA ordered target-column pair, together with its table context, to a scalar evidence value consumed by the local classifier. The initial operator library contains four broad categories: (i) \emph{profile and distributional operators}, which summarize value composition, length, cardinality, frequency, and entropy; (ii) \emph{structural pattern operators}, which capture recurring formats, value shapes, and domain-related patterns; (iii) \emph{contextual operators}, which characterize column position, table structure, and neighboring columns; and (iv) \emph{semantic similarity operators}, which measure the similarity between table evidence and the semantic concepts represented by the child nodes.

At the initial round $r=0$, every substrate adopts the default operator library, detailed in Appendix~\ref{app:operators_detail}. 
For an internal node $v$, the random forest $f_v^{(0)}$~\cite{breiman2001random} is trained on instances whose gold labels lie in the subtree rooted at $v$, with the immediate child of $v$ on the gold path used as the local routing target. In subsequent rounds, only substrates whose operator sets have been modified need to be reconstructed and retrained.
For an input instance $x$ and a candidate leaf label $y$ with root-to-leaf path
\[
\pi(y)=(v_0,v_1,\ldots,v_L=y),
\]
we compute its raw path score as
\[
s_y^{(r)}(x)
=
\prod_{\ell=1}^{L}
P\!\left(
v_\ell
\mid
v_{\ell-1},x;
f_{v_{\ell-1}}^{(r)}
\right),
\]
where each factor is the probability assigned by the substrate at parent node $v_{\ell-1}$ to child node $v_\ell$. Because paths may differ in depth and their local classifiers may produce differently scaled probabilities, we further apply a per-label isotonic calibration function $g_y$~\cite{niculescu2005predicting}  to produce label-ranking scores:
\[
\widetilde{s}_y^{(r)}(x)
=
g_y\!\left(s_y^{(r)}(x)\right),
\qquad
\widehat{y}
=
\arg\max_{y\in\mathcal{L}}
\widetilde{s}_y^{(r)}(x).
\]
Overall, each prediction remains traceable through its root-to-leaf semantic path, the local classifiers along that path, and the explicit operators used by those classifiers.
After training, each substrate is evaluated on its corresponding validation instances. We summarize its local decision accuracy, dominant confusions among child nodes, affected descendant labels, and representative misclassified instances into a substrate evidence profile $\mathcal{B}_v^{(r)}$ for generating substrate modifications.

\subsubsection{Substrate Modification}
\label{sec:m2:modification}
To address the unresolved semantic distinctions and evidence deficiencies revealed by the substrate evidence profile, substrate modification leverages the semantic reasoning and knowledge synthesis capabilities~\cite{wang2026llm,zhou2026multi} of an LLM to propose modifications to the current operator set of the substrate.
For each substrate $v$, the LLM receives three inputs: (i) the evidence profile $\mathcal{B}_v^{(r)}$; (ii) the current operator set $\mathcal{O}_v^{(r)}$; and (iii) the evolution history $\mathcal{H}_v^{(<r)}$, which records previously attempted modifications and their empirical outcomes.
Based on these inputs, the LLM contrasts the confused child groups, examines which evidence is already captured by the current operators, and hypothesizes generalizable signals that may better distinguish the corresponding semantic classes. Each proposed modification takes one of three forms: \emph{insert}, which introduces a missing operator; \emph{update}, which revises an existing operator; or \emph{remove}, which deletes a redundant or misleading operator. The resulting proposals form a substrate-specific modification queue:
\[
\mathcal{Q}_v^{(r)}
=
\operatorname{GenerateModifications}
\left(
\mathcal{B}_v^{(r)},
\mathcal{O}_v^{(r)},
\mathcal{H}_v^{(<r)}
\right).
\]
Each element $q\in\mathcal{Q}_v^{(r)}$ specifies one executable insert, update, or remove operation on the current operator set. A condensed prompt template is provided in Appendix~\ref{app:prompts}.

\subsubsection{Substrate Selection}
\label{sec:m2:selection}

LLM-generated modifications form a heterogeneous search space whose utility remains uncertain until retraining and empirical validation. To make effective use of a limited retraining budget and prioritize the most promising modifications, we perform exploration--exploitation-guided substrate selection. Specifically, the selection strategy balances exploring underexamined substrate-specific modification queues that may uncover previously overlooked evidence with exploiting queues whose prior modifications have yielded promising validation gains, avoid uniformly allocating resources to low-value proposals or prematurely concentrating the budget on a small subset of substrates. 

We therefore formulate substrate selection as a budgeted multi-armed bandit (MAB) problem~\cite{auer2002finite}, where each active substrate queue is treated as an arm and each selection evaluates its next modification using the substrate training and validation procedure described in Section~\ref{sec:m2:training}. Specifically, selecting arm $v$ nominates the next modification in $\mathcal{Q}_v^{(r)}$ for screening, and a gate-passing modification consumes one local training trial. Let $N_v$ denote the number of completed trials for substrate $v$, $\overline{R}_v$ its empirical mean validation reward, and $\tau$ the total number of completed trials. Active queues that have not yet received an evaluation opportunity are prioritized first. The remaining queues are selected using the following UCB-style score:
\[
v^{\star}
=
\arg\max_{v:\mathcal{Q}_v^{(r)}\neq\emptyset}
\left(
\overline{R}_v
+
\beta
\sqrt{
\frac{\log(\tau+1)}{N_v}
}
\right),
\]
where $\beta>0$ controls the exploration strength. The first term exploits substrates whose previous modifications have produced larger validation gains, whereas the second term explores substrates that have received fewer evaluation opportunities.

Once a substrate is selected, its next modification is screened for executability, relevance to the targeted distinction, and non-redundancy with the current operator set~\cite{peng2005feature}. Modifications rejected by these lightweight gates are discarded without consuming the local training budget. 
For a gate-passing modification $q$, let $\Delta a_{v,q}$ denote the change in local classification accuracy between the modified substrate and its current state. We assign reward $r_{v,q}=1+\alpha\Delta a_{v,q}$ when $\Delta a_{v,q}>0$ and $r_{v,q}=0$ otherwise, where $\alpha>0$ scales the accuracy-gain bonus. Thus, an effective modification receives a unit success reward plus its gain bonus, whereas neutral and risky modifications receive zero reward. Because local classification accuracy lies in $[0,1]$, the reward is bounded by $R_{\max}=1+\alpha$. Gate-passing outcomes update $\overline{R}_v$ and the corresponding evolution history.
The following lemma characterizes how the exploration term prevents the selection strategy from repeatedly concentrating on substrates with high historical rewards. Let $b_\tau(v)
=
\beta
\sqrt{
\frac{\log(\tau+1)}{N_v}
}$
denote the exploration bonus.
\begin{lemma}[Exploration dominance]
\label{lem:exploration_dominance}
Consider a selection after every active substrate queue has a positive recorded trial count. For any two active queues $v$ and $u$, if
$b_\tau(v)-b_\tau(u)>R_{\max}$, 
then
$\overline{R}_v+b_\tau(v)
>
\overline{R}_u+b_\tau(u)$.
Consequently, if this condition holds for every other active queue $u\neq v$, the selection strategy selects substrate $v$.
\end{lemma}

Lemma~\ref{lem:exploration_dominance} shows that a sufficiently underexplored substrate can outweigh any possible empirical-reward advantage of a more frequently evaluated substrate. 
The exploration term therefore helps prevent the limited evaluation budget from being persistently concentrated on only a small subset of modification queues. 
\section{Experimental Evaluation}
\subsection{Experimental Setup}

\begin{table}[t]
\centering
\caption{Dataset statistics. Prediction instances are target columns for CTA and ordered target-column pairs for CPA.}
\label{tab:dataset_stats}
\setlength{\tabcolsep}{4pt}
\renewcommand{\arraystretch}{1.05}
\resizebox{\columnwidth}{!}{
\begin{tabular}{llcccccc}
\toprule
\textbf{Alias} & \textbf{Dataset} & \textbf{Label} & \textbf{Task} & \textbf{Labels}
& \textbf{Train} & \textbf{Valid} & \textbf{Test} \\
\midrule
GTS-CTA       & GitTables & Schema.org & CTA & 59  & 499   & 93    & 98 \\
GTD-CTA       & GitTables & DBpedia    & CTA & 120 & 1{,}814 & 308 & 305 \\
ST-CTA     & SOTAB V2      & Schema.org           & CTA & 82  & 6{,}192 & 1{,}769 & 1{,}851 \\
T2D-CPA       & T2D V2        & DBpedia           & CPA  & 118 & 309   & 62    & 63 \\
ST-CPA  & SOTAB V2    & Schema.org         & CPA  & 108 & 8{,}603 & 2{,}458 & 2{,}340 \\
\bottomrule
\end{tabular}
}
\end{table}

\begin{table*}[t]
\centering
\caption{Overall performance by Micro-F1 and Macro-F1 (\%). Best results are in \textbf{bold}; second best are \underline{underlined}.}
\label{tab:overall_performance}
\setlength{\tabcolsep}{3.5pt}
\renewcommand{\arraystretch}{1.05}
\begin{tabular}{lcccccccccccc}
\toprule
& \multicolumn{2}{c}{GTD-CTA}
& \multicolumn{2}{c}{GTS-CTA}
& \multicolumn{2}{c}{ST-CTA}
& \multicolumn{2}{c}{T2D-CPA}
& \multicolumn{2}{c}{ST-CPA}
& \multicolumn{2}{c}{Average} \\
\cmidrule(lr){2-3} \cmidrule(lr){4-5} \cmidrule(lr){6-7} \cmidrule(lr){8-9} \cmidrule(lr){10-11} \cmidrule(lr){12-13}
Method & Micro & Macro & Micro & Macro & Micro & Macro & Micro & Macro & Micro & Macro & Micro & Macro \\
\midrule
TabEmb & 79.93 & 47.89 & 77.78 & 38.14 & 68.40 & 54.44 & 26.98 & 11.60 & 61.54 & 46.83 & 62.93 & 39.78 \\
Watchog & \underline{92.79} & \underline{63.59} & 38.78 & 28.85 & 76.99 & 71.44 & 38.10 & 18.43 & 70.90 & 59.46 & 63.51 & 48.35 \\
Doduo & 87.87 & 63.08 & \underline{90.82} & \underline{73.38} & 72.07 & 67.98 & \underline{82.54} & 64.10 & 75.98 & 71.10 & 81.86 & \underline{67.93} \\
REVEAL+ & 81.82 & 55.49 & 87.04 & 56.29 & 80.98 & 77.71 & \underline{82.54} & 63.68 & \underline{82.95} & \underline{77.76} & \underline{83.07} & 66.19 \\
\midrule
Deepseek (0-shot) & 36.82 & 33.96 & 63.27 & 46.41 & 79.54 & 77.27 & 77.78 & \textbf{75.36} & 71.75 & 64.94 & 65.83 & 59.59 \\
Deepseek (5-shot) & 59.30 & 49.76 & 68.37 & 52.17 & \underline{81.78} & \textbf{79.56} & 76.19 & 71.39 & 81.03 & 74.70 & 73.33 & 65.52 \\

\midrule
\methodName{} & \textbf{95.41} & \textbf{77.80} & \textbf{95.92} & \textbf{85.46} & \textbf{82.39} & \underline{79.51} & \textbf{90.48} & \underline{73.97} & \textbf{83.25} & \textbf{78.04} & \textbf{89.49} & \textbf{78.96} \\

\bottomrule
\end{tabular}
\vspace{-2mm}
\end{table*}

\begin{table*}[t]
\centering
\caption{Ablation studies for semantic structure and self-evolution by Micro-F1 and Macro-F1 (\%).}
\label{tab:ablations}
\setlength{\tabcolsep}{3.5pt}
\renewcommand{\arraystretch}{1.05}
\begin{tabular}{lcccccccccccc}
\toprule
& \multicolumn{2}{c}{GTD-CTA}
& \multicolumn{2}{c}{GTS-CTA}
& \multicolumn{2}{c}{ST-CTA}
& \multicolumn{2}{c}{T2D-CPA}
& \multicolumn{2}{c}{ST-CPA}
& \multicolumn{2}{c}{Average} \\
\cmidrule(lr){2-3} \cmidrule(lr){4-5} \cmidrule(lr){6-7} \cmidrule(lr){8-9} \cmidrule(lr){10-11} \cmidrule(lr){12-13}
Method & Micro & Macro & Micro & Macro & Micro & Macro & Micro & Macro & Micro & Macro & Micro & Macro \\
\midrule

\methodName{} & \textbf{95.41} & \textbf{77.80} & \textbf{95.92} & \textbf{85.46} & \textbf{82.39} & \textbf{79.51} & \textbf{90.48} & \textbf{73.97} & \textbf{83.25} & 78.04 & \textbf{89.49} & \textbf{78.96} \\
\methodName{} (w/o. skeleton) & 93.44 & 76.84 & 93.88 & 76.64 & 81.42 & 78.90 & 80.95 & 61.28 & 82.52 & \textbf{78.77} & 86.44 & 74.49 \\
\methodName{} (w/o. evolution) & 91.48 & 73.01 & 92.86 & 75.61 & 76.99 & 71.86 & \textbf{90.48} & \textbf{73.97} & 77.95 & 70.78 & 85.95 & 73.05 \\

\bottomrule
\end{tabular}
\vspace{-2mm}
\end{table*}

\noindent \textbf{Datasets.}
We evaluate \methodName{} on five public benchmark datasets covering both CTA and CPA tasks. Their names and statistics are summarized in Table~\ref{tab:dataset_stats}. \textit{GTD-CTA} and \textit{GTS-CTA} are two CTA benchmarks released as part of SemTab~\cite{abdelmageed2022results} and constructed from subsets of GitTables~\cite{hulsebos2023gittables}. Their columns are annotated using DBpedia and Schema.org type vocabularies, respectively.  From SOTAB V2~\cite{korini2022sotab,sotabv2}, we derive \textit{ST-CTA} and \textit{ST-CPA}, which provide Schema.org type annotations for individual columns and Schema.org property annotations for column pairs, respectively. For both benchmarks, we subsample the original training sets to a smaller scale while preserving full label coverage, and keep the official validation and test splits unchanged. We also construct \textit{T2D-CPA} for the CPA task from the attribute-to-property annotations of the T2D V2 gold standard~\cite{ritze2017matching}. Its tables originate from WebTables~\cite{cafarella2008webtables}, and its relation labels are drawn from DBpedia properties.

\noindent \textbf{Baselines.}
We compare \textit{\methodName{}} against representative CA baselines. \textit{Doduo}~\cite{suhara2022annotating} fine-tunes a pretrained Transformer for joint table representation learning and column annotation, while \textit{Watchog}~\cite{miao2023watchog} uses contrastive pretraining on unlabeled tables to learn transferable representations. We also evaluate direct LLM-based annotation with \textit{DeepSeek-v4-Pro}~\cite{xu2026deepseek} under zero-shot and five-shot prompting, following prior LLM-based CA studies~\cite{korini2023column,korini2024column}. \textit{\reveal{}}~\cite{ding2025retrieve} retrieves and verifies informative context columns before prediction, whereas \textit{\tabemb{}}~\cite{hoseinzade2026tabemb} combines LLM-derived semantic embeddings with graph-based structural modeling.

\noindent \textbf{Metrics.}
Following prior work~\cite{ding2025retrieve, miao2023watchog, suhara2022annotating}, we evaluate all methods using Micro-F1 and Macro-F1. Micro-F1 is computed globally over all predictions and reflects overall annotation accuracy, while Macro-F1 averages the per-class F1 scores and gives equal weight to each label, better reflecting performance under class imbalance.

\noindent \textbf{Implementation details.}
Module~1 generates $K=5$ candidate skeletons using DeepSeek-v4-Pro~\cite{xu2026deepseek}, with semantic similarity measured using \texttt{all-MiniLM-L6-v2}~\cite{wang2020minilm,reimers2019sentence} embeddings. Each internal node is trained as a random forest with 300 trees and may additionally incorporate \tabemb{}~\cite{hoseinzade2026tabemb} features if provided when selected through LLM-guided evolution. Local substrate evolution proceeds for five rounds, considering at most four operator modifications per substrate. A node-level UCB-style scheduler allocates a capped retraining budget of 6 trials, with the relevance and redundancy thresholds set to $\epsilon_v=0.01$ and $\gamma_v=0.92$, respectively. Experiments run on a workstation 
with an Intel(R) Xeon(R) Silver~4314 CPU~@~2.40\,GHz, an NVIDIA RTX~A5000 GPU, and 512\,GB RAM.

\vspace{-2mm}
\subsection{Overall Results. }
\label{sec:exp:overall}

\noindent \textbf{Exp-1: Overall annotation performance.}
Table~\ref{tab:overall_performance} summarizes the results on five CTA and CPA benchmarks. \methodName{} achieves the best Micro-F1 and either the best or second-best Macro-F1 on every dataset, yielding the highest average scores of 89.49\% and 78.96\%, respectively. In average Micro-F1, it surpasses \reveal{}, the strongest baseline, by 6.42\% and exceeds the remaining baselines by 7.63\%--26.56\%. Competing methods exhibit substantial modification across datasets: representation-learning approaches degrade under limited supervision, while direct prompting is sensitive to label semantics, as reflected by its particularly low performance on GTD-CTA. \methodName{} also shows a smaller Micro--Macro gap than Doduo and \reveal{}, indicating more balanced performance across frequent and infrequent labels. Although DeepSeek attains slightly higher Macro-F1 on ST-CTA and T2D-CPA, \methodName{} remains a close second on both and performs more consistently across the full benchmark suite. Overall, these results demonstrate the strong and robust annotation performance of \methodName{} across both tasks and label distributions.
We further evaluate the efficiency of \methodName{} in terms of runtime and LLM token consumption. On the two largest SOTAB benchmarks, \methodName{} is faster than \reveal{}, while using fewer tokens than direct prompting (details in Appendix~\ref{app:time_tokens}).

\noindent \textbf{Exp-2: Ablation studies.}
Table~\ref{tab:ablations} evaluates the contributions of the semantic skeleton and self-evolution in \methodName{}. Removing the hierarchy skeleton (\textit{w/o. skeleton}), which replaces hierarchical decision process with a single evolving classifier over all labels, decreases average Micro-F1 and Macro-F1 by 3.05\% and 4.47\%, respectively, with the largest degradation on T2D-CPA, confirming the value of semantic decomposition for distinguishing related labels. Removing self-evolution (\textit{w/o. evolution}), which uses the initially trained random forests without subsequent operator adaptation, decreases average Micro-F1 and Macro-F1 by 3.54\% and 5.91\%, respectively, demonstrating the benefit of refining node-specific evidence beyond the initial hierarchical predictor. The larger reductions in Macro-F1 further indicate that both components are important for fine-grained and infrequent labels. Although evolution brings no additional gain on T2D-CPA, likely due to its limited validation evidence for accepting local updates, the full model achieves the strongest average performance. Overall, the two modules are complementary: the skeleton organizes the label space into localized decision problems, while self-evolution adapts the evidence used to resolve them.
We also conduct an ablation study to evaluate the MAB-based scheduling strategy by comparing the UCB-style scheduler with random scheduling. The UCB-style scheduler achieves higher validation fitness with fewer retraining trials by prioritizing substrate queues whose previous modifications yield more promising outcomes, whereas random scheduling converges more slowly (detailed results in Appendix~\ref{app:mab_budget_allocation}).

\subsection{Robustness Analysis.}
\label{sec:exp:robustness}

\noindent \textbf{Exp-3: Varying LLMs.}
We evaluate the robustness of \methodName{} across three LLM backbones, Qwen3~\cite{yang2025qwen3}, MiniMax-M3~\cite{lai2026minimax}, and DeepSeek-V4-Pro~\cite{xu2026deepseek}, used for both global skeleton induction and local operator modification. As shown in Table~\ref{tab:llm_variants}, \methodName{} maintains consistently strong performance across all three backbones, with only modest modifications in Micro-F1 and Macro-F1. DeepSeek-V4-Pro performs best on the two ST benchmarks, while MiniMax-M3 achieves the strongest results on GTD-CTA. DeepSeek-V4-Pro nevertheless yields the best overall average performance and is therefore adopted as the default backbone.
\begin{table}[t]
\centering
\caption{Robustness to the LLM backbones. Values are test Micro-/Macro-F1 (\%).}
\label{tab:llm_variants}
\setlength{\tabcolsep}{2.8pt}
\renewcommand{\arraystretch}{1.05}
\resizebox{\columnwidth}{!}{
\begin{tabular}{lcccccc}
\toprule
& \multicolumn{2}{c}{GTD-CTA} & \multicolumn{2}{c}{ST-CTA} & \multicolumn{2}{c}{ST-CPA} \\
\cmidrule(lr){2-3} \cmidrule(lr){4-5} \cmidrule(lr){6-7}
\textbf{LLM backbones} & Micro & Macro
& Micro & Macro
& Micro & Macro \\
\midrule
Qwen3
& 94.75 & 74.41
& 80.93 & 78.27
& 81.11 & 75.12 \\
MiniMax-M3
& \textbf{95.41} & \textbf{79.85}
& 81.15 & 78.17
& 81.71 & 76.06 \\
DeepSeek V4 Pro
& \textbf{95.41} & 77.80
& \textbf{82.39} & \textbf{79.51}
& \textbf{83.25} & \textbf{78.04} \\
\bottomrule
\end{tabular}
}
\end{table}

\begin{figure}[t]
    \centering
    \begin{subfigure}[t]{0.23\textwidth}
        \centering
        \includegraphics[width=\linewidth]{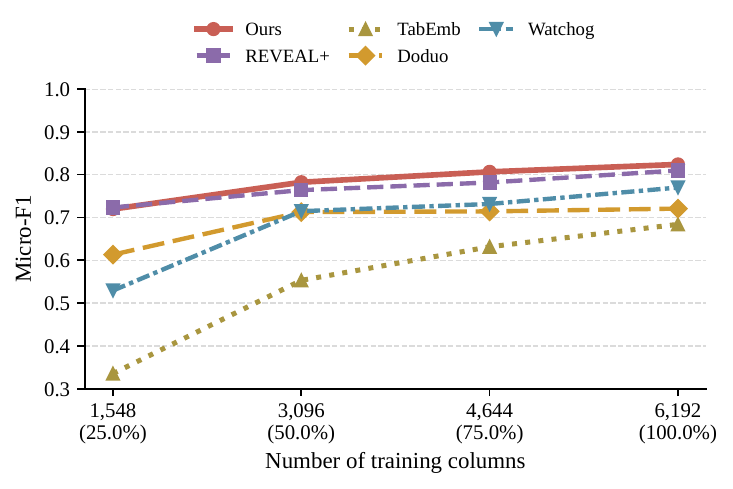}
        \caption{CTA Micro-F1}
        \label{fig:limited_training_cta}
    \end{subfigure}
    \hfill
    \begin{subfigure}[t]{0.23\textwidth}
        \centering
        \includegraphics[width=\linewidth]{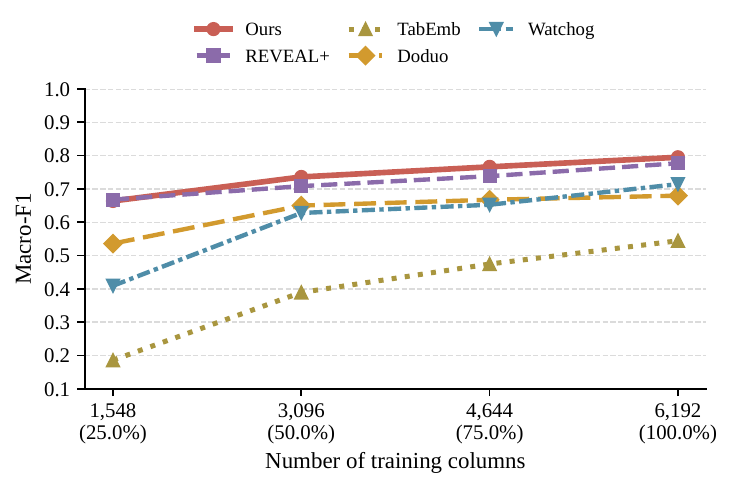}
        \caption{CTA Macro-F1}
        \label{fig:limited_training_cta_macro}
    \end{subfigure}
    
    \begin{subfigure}[t]{0.23\textwidth}
        \centering
        \includegraphics[width=\linewidth]{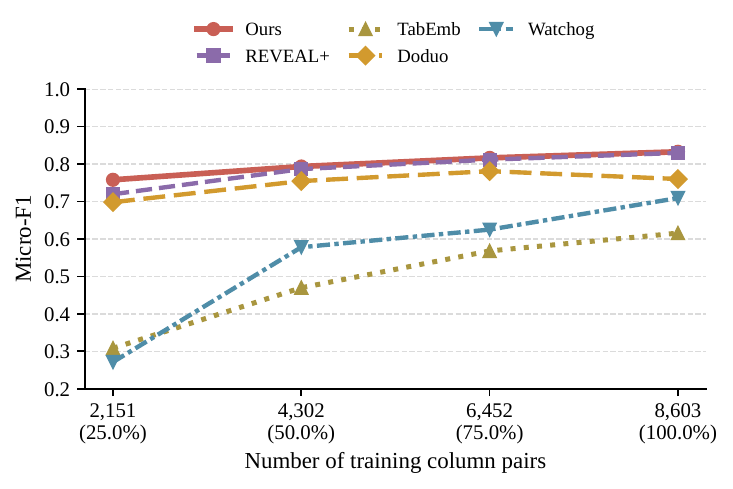}
        \caption{CPA Micro-F1}
        \label{fig:limited_training_re}
    \end{subfigure}
    \hfill
    \begin{subfigure}[t]{0.23\textwidth}
        \centering
        \includegraphics[width=\linewidth]{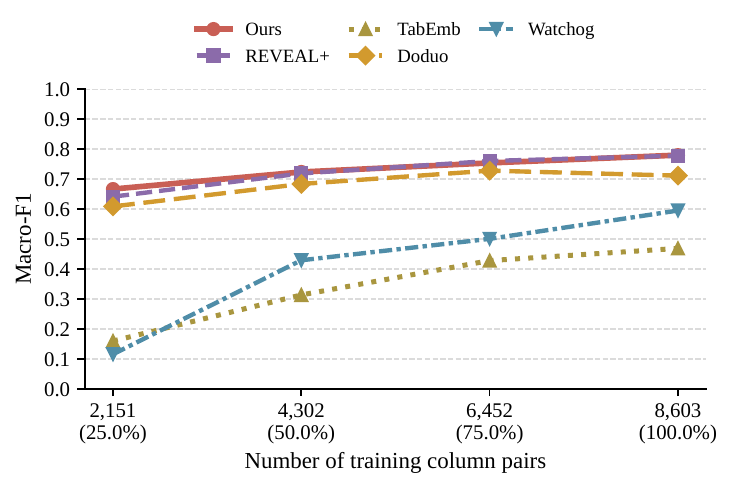}
        \caption{CPA Macro-F1}
        \label{fig:limited_training_re_macro}
    \end{subfigure}

    \caption{Performance under different amount of training data on ST-CTA and ST-CPA.}
    \label{fig:limited_training}
\end{figure}

\noindent \textbf{Exp-4: Varying amount of training data.}
We evaluate robustness by subsampling the ST-CTA and ST-CPA training sets to 25\%, 50\%, and 75\% while preserving full label coverage (Figure~\ref{fig:limited_training}). Representation-learning baselines are highly sensitive to reduced supervision: \tabemb{} and Watchog degrade sharply at the 25\% scale, particularly in Macro-F1. Doduo is more stable but consistently remains below the strongest methods, while \reveal{} is the most competitive baseline and closely follows \methodName{} across most settings. In contrast, \methodName{} degrades most gracefully, achieving the highest or near-highest Micro-F1 and Macro-F1 across both tasks and training scales. These results demonstrate its strong robustness under limited supervision.

\noindent \textbf{Exp-5: Varying number of evolution rounds.}
Figure~\ref{fig:evolution_rounds} reports test Micro-F1 and Macro-F1 across evolution rounds on ST-CTA and ST-CPA. On ST-CTA, both metrics increase steadily from the raw Round-0 model through five evolution rounds, with calibrated Micro-/Macro-F1 reaching 81.79\%/79.49\%. On ST-CPA, most improvements occur within the first two rounds, after which both metrics remain stable. Such non-monotonicity is expected because evolution states are selected using validation evidence rather than test performance. Overall, both tasks converge within five rounds, supporting the evolution round used in the main experiments.

\begin{figure}[t]
    \centering
    \begin{subfigure}[t]{0.23\textwidth}
        \centering
        \includegraphics[width=\linewidth]{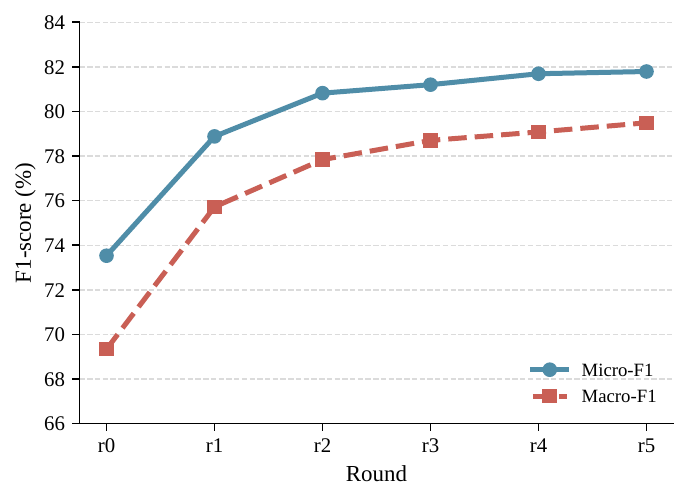}
        \caption{CTA on ST-CTA.}
        \label{fig:evolution_rounds_cta}
    \end{subfigure}
    \hfill
    \begin{subfigure}[t]{0.23\textwidth}
        \centering
        \includegraphics[width=\linewidth]{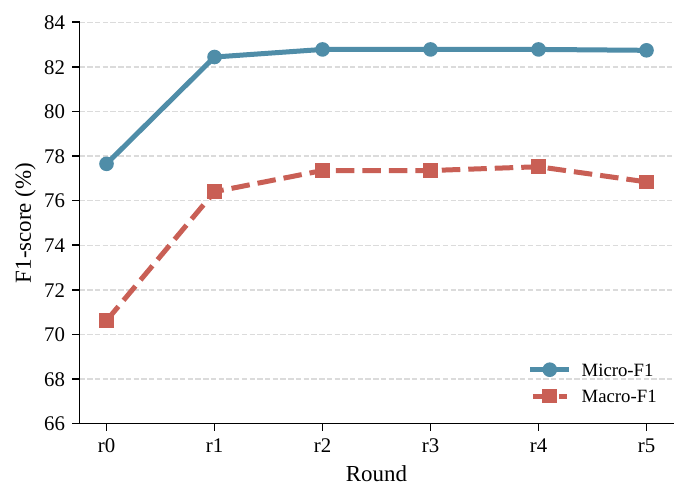}
        \caption{CPA on ST-CPA.}
        \label{fig:evolution_rounds_cpa}
    \end{subfigure}
    \caption{Evolution round-level test Micro-/Macro-F1.}
    \label{fig:evolution_rounds}
    \vspace{-2mm}
\end{figure}

\vspace{-2mm}

\begin{figure}[t]
    \centering
    \includegraphics[width=0.88\columnwidth]{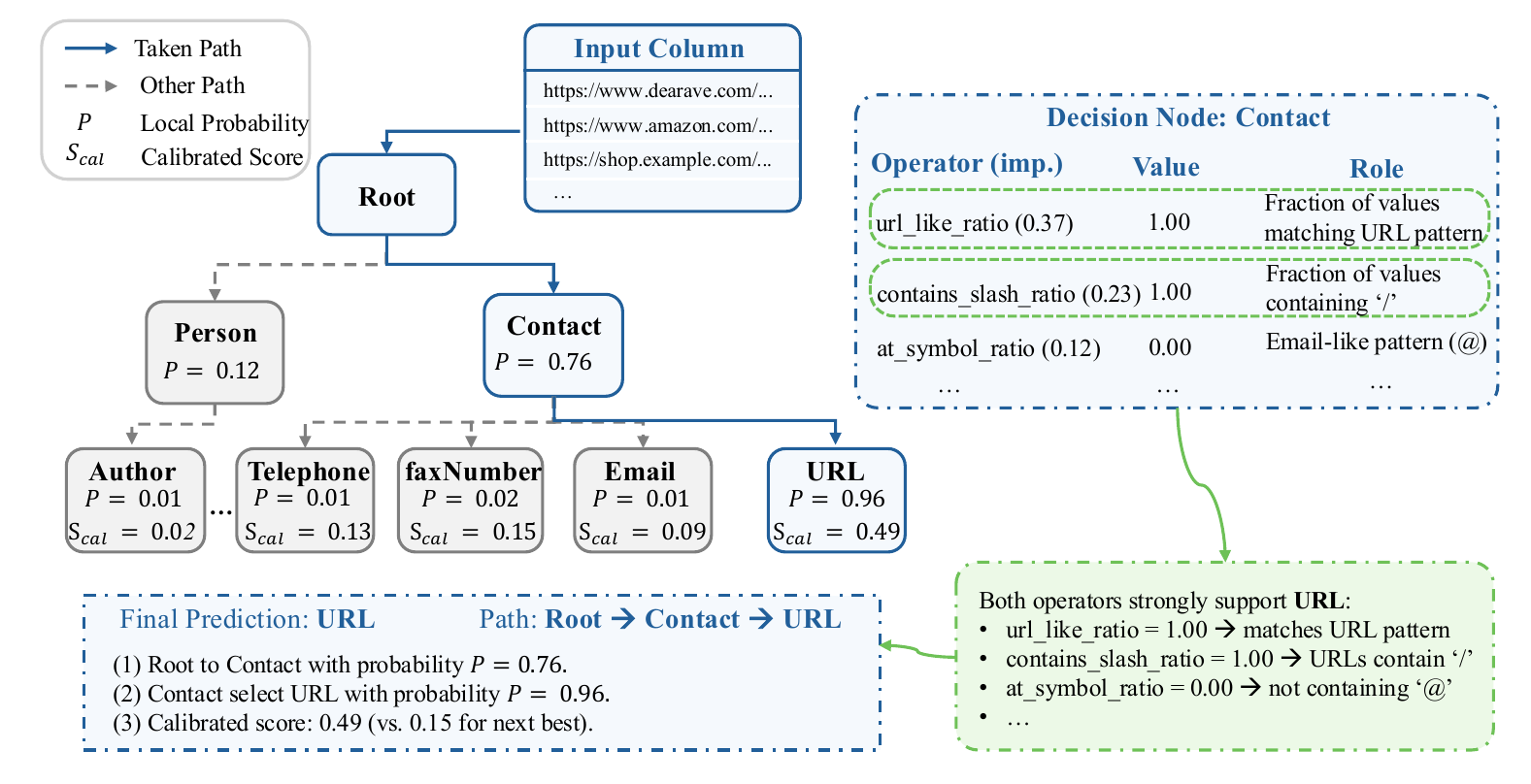}
    \caption{Example of CTA prediction. }
    \label{fig:case_study}
\end{figure}

\subsection{Interpretability Analysis.}
\noindent \textbf{Interpretable substrates.}
Each substrate exposes its node-specific symbolic evidence through named operators and their model-level importance, while each prediction remains traceable through its semantic path and the operator values observed at the local substrates. Table~\ref{tab:substrate_ops} lists the five most important operators at three representative GTD-CTA substrates. At \textit{Biology}, positional, contextual, similarity-based operators help distinguish taxonomic ranks from biological names. At \textit{Person}, contextual, statistical, and positional operators separate closely related person roles such as authors and directors. At \textit{Organization}, contextual operators, including neighboring-column statistics and table width, contribute most to distinguishing labels such as \textit{company} and \textit{manufacturer}.

\begin{table}[t]
\centering
\scriptsize
\setlength{\tabcolsep}{4pt}
\renewcommand{\arraystretch}{1.15}
\caption{Five most important operators at three representative GTD-CTA
substrates.}
\label{tab:substrate_ops}
\begin{tabular}{@{}p{1.7cm}lcr@{}}
\toprule
Substrate & Operator & Type & Importance \\
\midrule
\multirow{5}{=}{\textit{Biology}\\[2pt]}
 & \texttt{relative\_column\_index}        & position & .081 \\
 & \texttt{sim\_v\_Biology\_7}           & similarity   & .045 \\
 & \texttt{neighbor\_short\_text\_count}   & context  & .042 \\
 & \texttt{title\_case\_ratio}             & format   & .037 \\
 & \texttt{sim\_v\_Biology\_2}           & similarity   & .037 \\
\midrule
\multirow{5}{=}{\textit{Person}\\[2pt]}
 & \texttt{neighbor\_avg\_string\_length\_mean} & context  & .041 \\
 & \texttt{max\_string\_length}            & length   & .034 \\
 & \texttt{relative\_column\_index}        & position & .031 \\
 & \texttt{neighbor\_numeric\_like\_count} & context  & .030 \\
 & \texttt{neighbor\_short\_text\_count}   & context  & .028 \\
\midrule
\multirow{5}{=}{\textit{Organization}\\[2pt]}
 & \texttt{neighbor\_numeric\_like\_count}  & context & .042 \\
 & \texttt{neighbor\_unique\_ratio\_mean}   & context & .035 \\
 & \texttt{neighbor\_avg\_string\_length\_mean} & context & .034 \\
 & \texttt{neighbor\_numeric\_ratio\_mean}  & context & .034 \\
 & \texttt{table\_width}                    & context & .034 \\
\bottomrule
\end{tabular}
\end{table}

\noindent \textbf{Interpretable predictions.}
We trace an ST-CTA prediction to illustrate the interpretability of
\methodName{}. Consider a column containing product-page Web addresses, such as
\texttt{https://www.dearave...}, which is correctly
annotated as \textit{URL} in Figure~\ref{fig:case_study}. The prediction follows an explicit root-to-leaf path:
the root routes the column to \textit{Contact} with probability $0.76$, after
which the local classifier selects \textit{URL} with probability $0.96$. Its raw probability is $0.76*0.96$, while its calibrated score over all labels is $0.49$, compared with $0.15$ for the next candidate.
The deciding \textit{Contact} node distinguishes \textit{URL}, \textit{telephone}, \textit{faxNumber}, and \textit{email} using explicit operators. Its two highest-ranked operators are \texttt{url\_like\_ratio} and \texttt{contains\_slash\_ratio}, both of which evaluate to $1.0$ on this column, whereas phone-like values do not activate the URL-pattern operator. The prediction is inspectable through its semantic path, the node-specific operators, their observed values, and their model-level importance.

\section{Related Work}
\label{sec:Relatedwork}

Column annotation in tabular data has evolved through three major paradigms. Early approaches perform CA using handcrafted rules, patterns, keywords, and manually designed lexical, statistical, or structural features~\cite{ramnandan2015assigning, pham2016semantic, hulsebos2019sherlock,zhang2019sato}. While effective in specific domains, they require substantial expert effort and are limited by the coverage and transferability of the predefined feature space. Transformer-based methods instead learn representations of table content with pretrained language models and fine-tune downstream CA predictors~\cite{deng2022turl, suhara2022annotating, miao2023watchog, ding2025retrieve, hoseinzade2026tabemb}. While more generalizable, they underutilize the semantic structure of the label space, and the prediction process is opaque. More recently, LLM-based methods directly predict labels through prompting~\cite{feuer2024archetype, korini2023column, korini2024column}. However, processing large tables is constrained by context windows, while relying mainly on observed values remains insufficient for distinguishing semantically related labels.


\section{Conclusion}
We presented \methodName{}, an LLM-empowered interpretable framework for CTA and CPA that materializes column annotation as a global-to-local symbolic decision process. It first induces a robust hypernym-inspired semantic skeleton over the label space through MBR-based consensus, and then materializes each internal node as an executable and evolvable predictive substrate through LLM-guided operator modification and exploration--exploitation-based selection. Experiments across multiple benchmarks demonstrate that \methodName{} is accurate, robust, and interpretable.

\clearpage
\bibliographystyle{ACM-Reference-Format}
\bibliography{ref}  

\balance
\clearpage
\appendix

\section{LLM Prompt Templates}
\label{app:prompts}

\begin{tcolorbox}[
    enhanced,
    sharp corners,
    boxrule=0.5pt,
    colback=white,
    colframe=black,
    label={box:prompt_m1},
    title=\textbf{Prompt Design A: LLM-Guided Skeleton Induction},
    fonttitle=\bfseries\sffamily,
    coltitle=black,
    attach boxed title to top left={yshift=-2mm, xshift=2mm},
    boxed title style={colback=white, colframe=white},
    drop shadow
]
\small

\begin{enumerate}[label=(\arabic*), itemsep=0pt, parsep=2pt, leftmargin=*]

    \item \textbf{Instruction:}  
    Organize the given CTA label vocabulary into a hypernym-inspired semantic
    tree that supports hierarchical column-type classification.

    \item \textbf{Input:}
    \begin{itemize}[label=$\bullet$, itemsep=0pt, leftmargin=*]
        \item The complete target label vocabulary $\mathcal{L}$;
        \item No column values or table-specific evidence.
    \end{itemize}

    \item \textbf{Guidelines:}
    \begin{itemize}[label=$\bullet$, itemsep=0pt, leftmargin=*]
        \item Construct the hierarchy solely from label semantics.
        \item Use meaningful hypernym concepts as internal nodes.
        \item Group broad property types near the root and introduce
        entity-specific groupings only when semantically appropriate.
        \item Include every input label exactly once as a leaf.
        \item Avoid unary branches, arbitrary capacity-based groups, and
        semantically uninformative node names.
        \item Restrict the tree to a maximum depth of $3$ and maximum branching
        factor of $8$.
    \end{itemize}

    \item \textbf{Output:}  
    Return one machine-readable JSON tree in which internal nodes contain
    \texttt{name} and \texttt{children}, and leaves contain the original
    \texttt{name} and \texttt{label} verbatim.

\end{enumerate}
\end{tcolorbox}
\begin{tcolorbox}[
    enhanced,
    sharp corners,
    boxrule=0.5pt,
    colback=white,
    colframe=black,
    label={box:prompt_m2},
    title=\textbf{Prompt Design B: LLM-Guided Operator Modification},
    fonttitle=\bfseries\sffamily,
    coltitle=black,
    attach boxed title to top left={yshift=-2mm, xshift=2mm},
    boxed title style={colback=white, colframe=white},
    drop shadow
]
\small

\begin{enumerate}[label=(\arabic*), itemsep=0pt, parsep=2pt, leftmargin=*]

    \item \textbf{Instruction:}  
    Diagnose a local child-routing confusion and propose interpretable operator
    modifications that provide the missing discriminative evidence.

    \item \textbf{Input:}
    \begin{itemize}[label=$\bullet$, itemsep=0pt, leftmargin=*]
        \item The substrate evidence profile, including routing accuracy,
        confused child pairs, and representative misclassified values;
        \item The current node-specific operator set;
        \item Previous effective, neutral, and risky modifications from the
        evolution memory.
    \end{itemize}

    \item \textbf{Guidelines:}
    \begin{itemize}[label=$\bullet$, itemsep=0pt, leftmargin=*]
        \item Contrast the empirical patterns of misrouted and correctly routed
        instances.
        \item Combine the observed evidence with prior knowledge of how the
        corresponding semantic types are commonly represented.
        \item Propose \texttt{add}, \texttt{update}, or \texttt{delete}
        modifications rather than directly modifying model parameters.
        \item Express each proposed operator as a scalar-valued
        \texttt{feature(values)} function.
        \item Use only the permitted DSL operations and avoid imports, file
        access, dynamic execution, and unsafe attributes.
        \item Prefer generalizable semantic patterns over rules that memorize
        the displayed examples.
    \end{itemize}

    \item \textbf{Output:}  
    Return a machine-readable JSON object containing the diagnosed semantic
    distinction, the supporting hypothesis, and one or more executable operator
    modifications with their operation type, name, specification, and code.

\end{enumerate}

\end{tcolorbox}

For reproducibility, we provide condensed sample prompts for the two LLM calls used in \methodName{}. Long label lists, representative samples, and repeated instructions are abbreviated, while the original objectives, constraints, and output schemas are preserved.

Prompt Design~A presents the global skeleton-induction prompt. The LLM receives only the target label vocabulary and organizes it into a hypernym-inspired semantic tree, without access to table values. Structural constraints require every label to appear exactly once as a leaf, ensuring that each internal node defines a valid partition of its semantic subspace.

Prompt Design~B presents the local operator modification prompt. The LLM receives a substrate evidence profile, the current operator set, and prior evolution outcomes, and then contrasts confused child groups to identify missing discriminative evidence. Proposed additions, revisions, or removals follow the restricted DSL and be expressed as sandbox-executable scalar operators. The resulting modifications are subsequently evaluated through gated empirical validation before being incorporated into the corresponding substrate.

\section{Default Operator Library}
\label{app:operators_detail}

At $r=0$, every substrate is initialized with the same deterministic default library; subsequent evolution makes its operator set node-specific. The library is organized into three shared families. \emph{Value-profile and distributional operators} summarize the target column through its number of values, missingness, uniqueness and duplication, string-length and token-count statistics, character composition, value entropy, concentration, and rarity. \emph{Surface-pattern operators} measure the fractions of values matching common semantic formats, including numeric, date, year, URL, email, phone, ISBN, capitalization, punctuation, and fixed-width identifier patterns. \emph{Contextual operators} encode table width, relative column position, and aggregate profile statistics of neighboring columns. All three families are scalar, transparent functions of the table contents and are shared across substrates.
Each substrate additionally receives \emph{node-specific semantic similarity operators}. They compare the target-column values and its neighboring-column context with TF--IDF centroids constructed from the training instances assigned to each child semantic branch. For CPA, the library further includes profiles of both columns in the ordered pair, their value overlap, token-set similarity, relative length and position, and comparisons with same-shape sibling columns.

\section{Algorithmic Summary}
\label{app:algorithms}

Algorithms~\ref{alg:global_skeleton_induction} and~\ref{alg:local_substrate_evolution} summarize the two modules using the notation in Sections~\ref{sec:module1} and~\ref{sec:module2}. Algorithm~\ref{alg:global_skeleton_induction} separates individual candidate assessment from consensus: valid LLM outputs form $\mathbb{T}$, their semantic--structural costs induce the reliability weights $w_k$, and MBR selects an existing candidate with minimum quality-weighted structural disagreement.

\begin{algorithm}[h]
\small
\caption{Module~1: Global Skeleton Induction}
\label{alg:global_skeleton_induction}
\LinesNumbered
\DontPrintSemicolon
\KwIn{label set $\mathcal{L}$, LLM $f_{\mathrm{LLM}}$, candidate count $K$}
\KwOut{consensus skeleton $\mathcal{T}^{\star}$}
\For{$k\leftarrow 1$ \KwTo $K$}{
  $p^{(k)}\leftarrow\operatorname{PromptVariant}(\mathcal{L},k)$\;
  $\mathcal{T}^{(k)}\leftarrow\operatorname{Parse}\!\left(f_{\mathrm{LLM}}(p^{(k)})\right)$\;
  \If{$\operatorname{Valid}(\mathcal{T}^{(k)},\mathcal{L})$}{
    $\mathbb{T}\leftarrow\mathbb{T}\cup\{\mathcal{T}^{(k)}\}$\;
  }
}
\ForEach{$\mathcal{T}^{(k)}\in\mathbb{T}$}{
  $c_k\leftarrow\mathrm{Cost}_{\mathrm{sem}}(\mathcal{T}^{(k)})+\mathrm{Cost}_{\mathrm{str}}(\mathcal{T}^{(k)})$\;
}
$w_k\leftarrow\exp(-\eta c_k)\,/\,\sum_{\mathcal{T}^{(j)}\in\mathbb{T}}\exp(-\eta c_j)$\quad$\forall\mathcal{T}^{(k)}\in\mathbb{T}$\;
$\displaystyle\mathcal{T}^{\star}\leftarrow\arg\min_{\mathcal{T}^{(k)}\in\mathbb{T}}\sum_{\mathcal{T}^{(j)}\in\mathbb{T}}w_j\,d_{\mathrm{tree}}\!\left(\mathcal{T}^{(k)},\mathcal{T}^{(j)}\right)$\;
\Return{$\mathcal{T}^{\star}$}\;
\end{algorithm}


\begin{algorithm}[h]
\small
\caption{Module~2: Local Substrate Evolution}
\label{alg:local_substrate_evolution}
\LinesNumbered
\DontPrintSemicolon
\KwIn{$\mathcal{T}^{\star}$, $\mathcal{D}_{\mathrm{tr}}$,
$\mathcal{D}_{\mathrm{va}}$, $f_{\mathrm{LLM}}$, rounds $R$, budget $B$}
\KwOut{selected substrate states $\{\mathcal{S}_v^{(r^{\star})}\}$}

$\{\mathcal{S}_v^{(0)},\mathcal{H}_v\}_{v\in\mathcal{V}_{\mathrm{int}}}
\leftarrow
\operatorname{Initialize}(\mathcal{T}^{\star},\mathcal{D}_{\mathrm{tr}})$\;

\For{$r\leftarrow0$ \KwTo $R-1$}{
  $\{\mathcal{S}_v^{(r+1)}\}\leftarrow\{\mathcal{S}_v^{(r)}\}$\;
  $(\mathcal{V}_{\mathrm{act}},\{\mathcal{Q}_v\})
  \leftarrow
  \operatorname{DiagnoseAndPropose}
  (\{\mathcal{S}_v^{(r)}\},\mathcal{D}_{\mathrm{va}},
  f_{\mathrm{LLM}},\{\mathcal{H}_v\})$\;

  \For{$b\leftarrow1$ \KwTo $B$}{
    \If{$\forall v,\mathcal{Q}_v=\emptyset$}{\textbf{break}\;}
    $(v,q)\leftarrow
    \operatorname{UCBSelectAndPop}
    (\{\mathcal{Q}_v\},\{N_v,\overline{R}_v\},\tau)$\;

    \If{$\operatorname{Gate}(q,\mathcal{S}_v^{(r+1)})$}{
      $(\widetilde{\mathcal{S}}_v,\Delta a)
      \leftarrow
      \operatorname{RetrainAndValidate}
      (\mathcal{S}_v^{(r+1)},q,
      \mathcal{D}_{\mathrm{tr}},\mathcal{D}_{\mathrm{va}})$\;
      $(\mathcal{S}_v^{(r+1)},N_v,\overline{R}_v,\tau,\mathcal{H}_v)
      \leftarrow
      \operatorname{Update}
      (\mathcal{S}_v^{(r+1)},\widetilde{\mathcal{S}}_v,
      N_v,\overline{R}_v,\tau,\mathcal{H}_v,q,\Delta a)$\;
    }
  }
}

$r^{\star}\leftarrow
\arg\max_{0\leq r\leq R}
\operatorname{Fitness}_{\mathrm{va}}
(\{\mathcal{S}_v^{(r)}\})$\;
\Return{$\{\mathcal{S}_v^{(r^{\star})}\}$}\;
\end{algorithm}

Given the fixed $\mathcal{T}^{\star}$, Algorithm~\ref{alg:local_substrate_evolution} first materializes every internal node at $r=0$. Each subsequent round diagnoses active substrates, constructs their modification queues, and allocates at most $B$ gate-passing retraining trials. The scheduler gives untried queues initial coverage and then uses $\overline{R}_v+b_\tau(v)$ to balance prior gains and exploration; the validation-selected state is returned after $R$ rounds.

\section{Proof}
\subsection{Proof of Lemma~\ref{prop:mbr_medoid}}
\label{app:mbr_medoid}

 \begin{proof}
  Let $p\geq\rho$ be the weighted mass within distance $\delta$ of $\bar{\mathcal{T}}$. Splitting the risk of $\bar{\mathcal{T}}$ into this neighborhood and its complement gives
  \[
  \mathrm{Risk}_{\eta}(\bar{\mathcal{T}})\leq p\delta+(1-p)D_{\mathbb{T}}\leq\rho\delta+(1-\rho)D_{\mathbb{T}},
  \]
  where the second inequality uses $p\geq\rho$ and $\delta\leq D_{\mathbb{T}}$. Since $\mathcal{T}^{\star}$ minimizes $\mathrm{Risk}_{\eta}$ over $\mathbb{T}$, the same upper bound holds for $\mathrm{Risk}_{\eta}(\mathcal{T}^{\star})$, proving
  part~(i). For part~(ii), let $\Delta=d_{\mathrm{tree}}(\mathcal{T}',\bar{\mathcal{T}})$. Every candidate in the $\delta$-neighborhood of $\bar{\mathcal{T}}$ is at least $\Delta-\delta$ away from $\mathcal{T}'$ by the triangle inequality. Hence
  \[
  \mathrm{Risk}_{\eta}(\mathcal{T}')\geq\rho(\Delta-\delta).
  \]
  When $\Delta>2\delta+\frac{1-\rho}{\rho}D_{\mathbb{T}}$, this lower bound exceeds $\rho\delta+(1-\rho)D_{\mathbb{T}}$, which is an upper bound on $\mathrm{Risk}_{\eta}(\mathcal{T}^{\star})$. Therefore, $\mathcal{T}'$ cannot minimize
  Eq.~\eqref{eq:mbr_medoid}.
  \end{proof}

\subsection{Proof of Lemma~\ref{lem:exploration_dominance}}
\label{app:exploration_dominance}

\begin{proof}
Since every reward lies in $[0,R_{\max}]$, the empirical means satisfy $\overline{R}_v-\overline{R}_u\geq -R_{\max}$. Therefore,
\[
\bigl(\overline{R}_v+b_\tau(v)\bigr)-\bigl(\overline{R}_u+b_\tau(u)\bigr)
=
\overline{R}_v-\overline{R}_u+b_\tau(v)-b_\tau(u)
\geq
-R_{\max}+b_\tau(v)-b_\tau(u).
\]
When $b_\tau(v)-b_\tau(u)>R_{\max}$, the right-hand side is positive, proving that the score of $v$ exceeds that of $u$. If the inequality holds for every other active queue, $v$ attains the maximal scheduling score and is selected.
\end{proof}

\section{Additional Experimental Evaluations}
\noindent \textbf{Exp-6: Runtime and token efficiency.} \label{app:time_tokens} Figures~\ref{fig:time} and~\ref{fig:tokens} compare runtime and LLM token consumption, respectively. As shown in Figure~\ref{fig:time}, \methodName{} is consistently faster than \reveal{}, \tabemb{}, and direct LLM prompting, with the runtime advantage becoming more pronounced on the larger ST-CTA and ST-CPA benchmarks. Although Doduo remains faster, its lower runtime is accompanied by substantially weaker annotation performance (Table~\ref{tab:overall_performance}).Unlike direct prompting, which consumes tokens for every prediction, \methodName{} uses LLMs only during skeleton induction and substrate evolution. Figure~\ref{fig:tokens} shows that this design leads to lower token consumption on the larger ST benchmarks, especially compared with five-shot prompting. Overall, \methodName{} provides a favorable balance among runtime, LLM cost, and annotation accuracy.

\begin{figure}[t]
    \centering
    \begin{subfigure}[t]{0.23\textwidth}
        \centering
        \includegraphics[width=\linewidth]{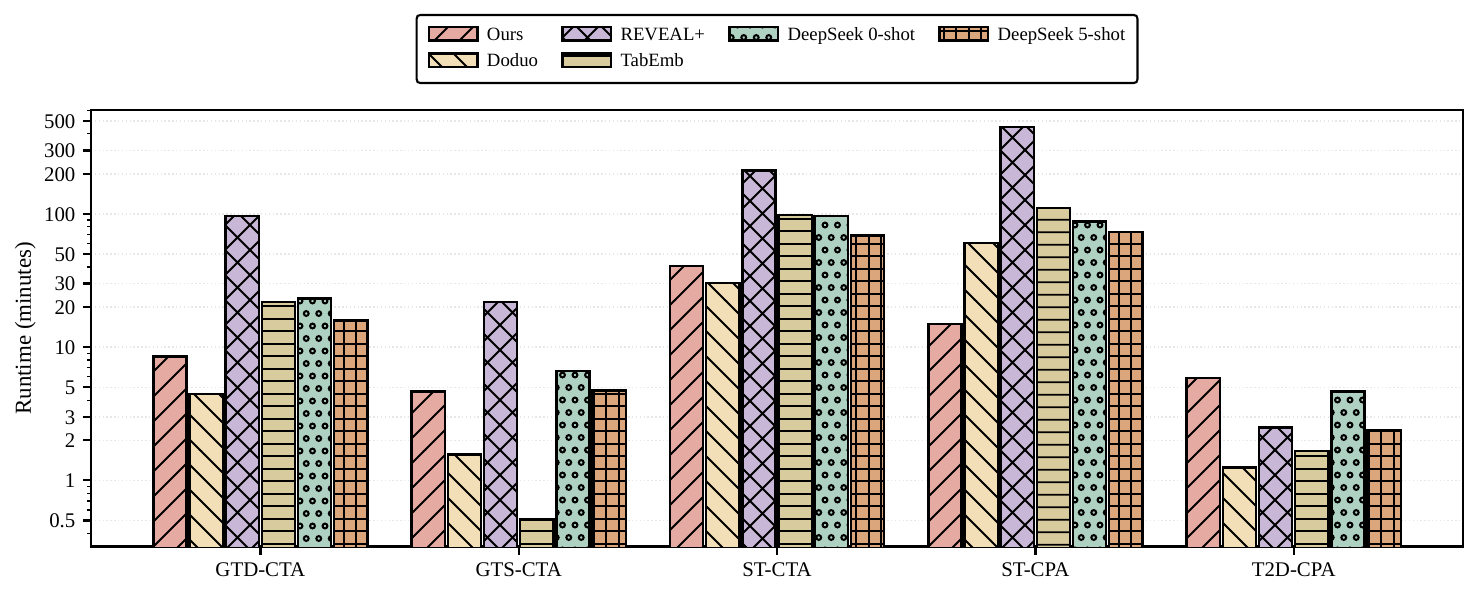}
        \caption{Runtime.}
        \label{fig:time}
    \end{subfigure}
    \hfill
    \begin{subfigure}[t]{0.23\textwidth}
        \centering
        \includegraphics[width=\linewidth]{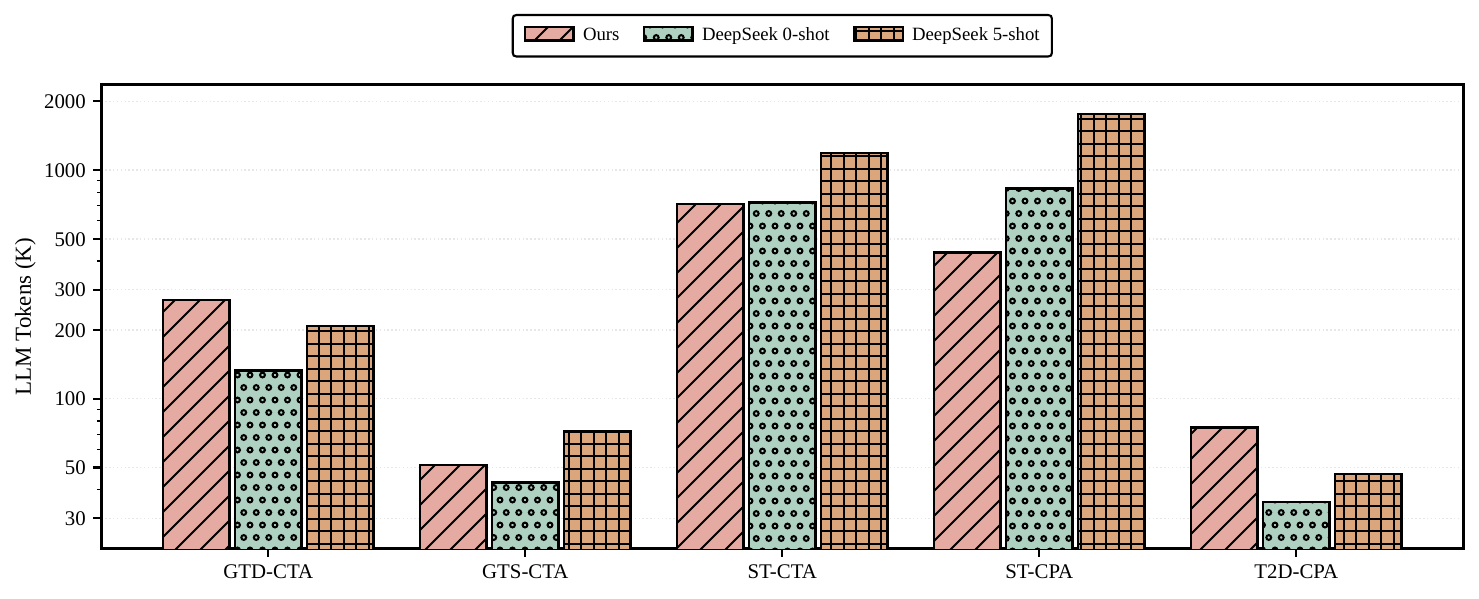}
        \caption{LLM token consumption.}
        \label{fig:tokens}
    \end{subfigure}
    \caption{Efficiency comparison across CTA and CPA datasets.}
    \label{fig:efficiency}
\end{figure}

\begin{figure}[t]
    \centering
    \includegraphics[width=0.4\columnwidth]{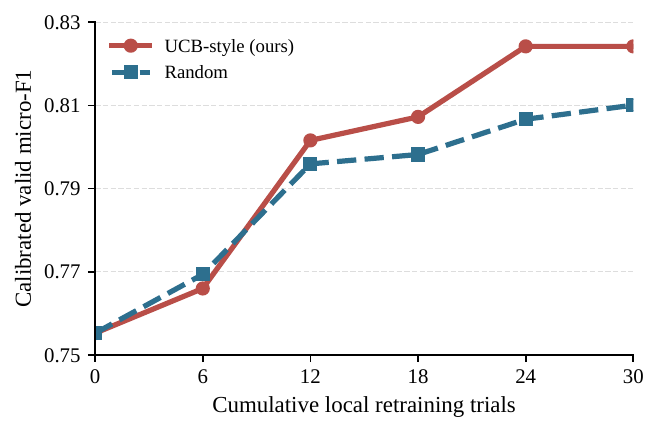}
    \caption{MAB-based validation-budget allocation on ST-CTA. Both schedulers receive six local retraining trials per round. The UCB-style scheduler reaches the same calibrated validation Micro-F1 with fewer trials than random scheduling.}
    \label{fig:mab_budget_allocation}
\end{figure}

\noindent \textbf{Exp-7: Effectiveness of Budgeted Explore--Exploit Selection.}
\label{app:mab_budget_allocation}
We isolate the contribution of the budgeted explore--exploit selection strategy in local substrate evolution by comparing the UCB-based scheduler with random selection on ST-CTA. Both methods receive the same budget of six substrate--proposal validation trials per evolution round. Figure~\ref{fig:mab_budget_allocation} reports the calibrated out-of-fold validation Micro-F1 used for selecting the inherited substrate state. The UCB-based scheduler exceeds 0.80 after 12 trials, whereas random selection requires 24 trials to reach the same level. After 30 trials, the two methods achieve 0.8242 and 0.8101, respectively. These results show that the proposed scheduler improves validation efficiency under a fixed retraining budget by allocating trials to more promising substrate based on prior outcomes.

\end{document}